\documentclass[reqno]{amsart}

\usepackage{amssymb}
\usepackage{graphicx}
\usepackage{hyperref}

\DeclareMathOperator{\tr}{tr}

\DeclareMathOperator{\rank}{rank}

\theoremstyle{plain}
\newtheorem{theorem}{Theorem}
\newtheorem{lemma}{Lemma}

\theoremstyle{remark}
\newtheorem{remark}{Remark}

\begin{document}

\title{On Some Properties of Calibrated Trifocal Tensors}

\author{E.V. Martyushev}

\date{May 15, 2016}


\address{South Ural State University, 76 Lenin Avenue, Chelyabinsk 454080, Russia}
\email{mev@susu.ac.ru}

\begin{abstract}

In two-view geometry, the essential matrix describes the relative position and orientation of two calibrated images. In three views, a similar role is assigned to the calibrated trifocal tensor. It is a particular case of the (uncalibrated) trifocal tensor and thus it inherits all its properties but, due to the smaller degrees of freedom, satisfies a number of additional algebraic constraints. Some of them are described in this paper. More specifically, we define a new notion --- the trifocal essential matrix. On the one hand, it is a generalization of the ordinary (bifocal) essential matrix, and, on the other hand, it is closely related to the calibrated trifocal tensor. We prove the two necessary and sufficient conditions that characterize the set of trifocal essential matrices. Based on these characterizations, we propose three necessary conditions on a calibrated trifocal tensor. They have a form of 15 quartic and 99 quintic polynomial equations. We show that in the practically significant real case the 15 quartic constraints are also sufficient.

\end{abstract}

\maketitle

\section{Introduction}

In multiview geometry, the fundamental matrix and the trifocal tensor describe the relative orientation of two and three (uncalibrated) images respectively. If the cameras are pre-calibrated, i.e. we are given the calibration matrices for each view, the fundamental matrix is transformed to the so-called essential matrix. It was first introduced by Longuet-Higgins in~\cite{LH81}. The essential matrix has fewer degrees of freedom and additional algebraic properties, compared to the fundamental matrix. A detailed investigation of these properties is given by Demazure, Faugeras, Maybank and other researchers in~\cite{Dem88}, \cite{FM}, \cite{Horn}, \cite{HF89}, \cite{Maybank}. We shortly recall the most important of them in the next section.

The trifocal tensor for calibrated cameras (we call this entity the calibrated trifocal tensor) was first appeared in the papers by Spetsakis and Aloimonos~\cite{SA} and Weng, Huang and Ahuja~\cite{Weng}. Later, Hartley~\cite{Hartley} generalized the trifocal tensor for the case of uncalibrated cameras. The properties of the (uncalibrated) trifocal tensors and their characterizations have been investigated by Hartley, Shashua, Triggs and other researchers in~\cite{Faugeras}, \cite{Shashua}, \cite{Triggs}, \cite{VL93}.

As well as the essential matrix, the calibrated trifocal tensor has fewer degrees of freedom and additional algebraic properties, compared to the uncalibrated case. The investigation of these properties is the main purpose of the present paper. In particular, we show that the calibrated trifocal tensor must satisfy a number of low degree homogeneous polynomial equations. These equations arise from the characterization constraints on a certain complex matrix associated with a calibrated trifocal tensor.

The results of the paper can be applied to different computer vision problems, such as metric scene reconstruction, camera self-calibration, bundle adjustment, etc.

The rest of the paper is organized as follows. In Section~\ref{sec:prel}, we recall some definitions and results from multiview geometry. In Section~\ref{sec:character}, we introduce a new notion --- the trifocal essential matrix. On the one hand, it is a generalization of the ordinary (bifocal) essential matrix, and, on the other hand, it is closely related to the calibrated trifocal tensor. We prove the two necessary and sufficient conditions that characterize the set of trifocal essential matrices. In Section~\ref{sec:conds}, we define the trifocal essential matrix associated with a calibrated trifocal tensor and give its geometric interpretation. Based on the characterizations from Section~\ref{sec:character}, we propose our three necessary conditions. They have a form of 15 quartic and 99 quintic polynomial equations in the entries of a calibrated trifocal tensor. In Section~\ref{sec:suff}, we show that in the practically significant real case the 15 quartic constraints are also sufficient. In Section~\ref{sec:disc}, we discuss the results of the paper.

\section{Preliminaries}
\label{sec:prel}

\subsection{Notation}

We preferably use $\alpha, \beta, \ldots$ for scalars, $a, b, \ldots$ for column 3-vectors or polynomials, and $A, B, \ldots$ both for matrices and column 4-vectors. For a matrix~$A$ the entries are $(A)_{ij}$, the transpose is~$A^{\mathrm T}$, the determinant is $\det A$, and the trace is~$\tr A$. For two 3-vectors~$a$ and~$b$ the cross product is $a\times b$. For a vector~$a$ the entries are $(a)_i$, the notation $[a]_\times$ stands for the skew-symmetric matrix such that $[a]_\times b = a \times b$ for any vector~$b$. We use~$I$ for identical matrix.

The group of $n\times n$ matrices subject to $RR^{\mathrm T} = I$ and $\det R = 1$ is denoted by $\mathrm{SO}(n)$ in case~$R$ is real and $\mathrm{SO}(n, \mathbb C)$ if~$R$ is allowed to have complex entries.

\subsection{Pinhole cameras}

We briefly recall some definitions and results from multiview geometry, see~\cite{FM}, \cite{Faugeras93}, \cite{HZ}, \cite{Maybank} for details.

A \emph{pinhole camera} is a triple $(O, \Pi, P)$, where $\Pi$ is the image plane, $P$ is a central projection of points in three-dimensional Euclidean space onto~$\Pi$, and~$O \not\in \Pi$ is the camera centre (centre of projection~$P$).

Let there be given coordinate frames in 3-space and in the image plane~$\Pi$. Let~$Q$ be a point in 3-space represented in homogeneous coordinates as a 4-vector, and~$q$ be its image in~$\Pi$ represented as a 3-vector. Projection~$P$ is then given by a $3\times 4$ homogeneous matrix, which is called the \emph{camera matrix} and is also denoted by~$P$. We have
\[
q \sim PQ,
\]
where $\sim$ means an equality up to a scale. For the sake of brevity, we identify further the camera $(O, \Pi, P)$ with its camera matrix~$P$.

The \emph{focal length} is the distance between~$O$ and~$\Pi$, the orthogonal projection of~$O$ onto~$\Pi$ is called the \emph{principal point}. All intrinsic parameters of a camera (such as the focal length, the principal point offsets, etc.) are combined into a single upper-triangular matrix, which is called the \emph{calibration matrix}. A camera is called \emph{calibrated} if its calibration matrix is known.

By changing coordinates in the image plane, the calibrated camera can be represented in form
\[
P = \begin{bmatrix}R & t\end{bmatrix},
\]
where $R \in \mathrm{SO}(3)$ is called the \emph{rotation matrix} and $t \in \mathbb R^3$ is called the \emph{translation vector}.

\subsection{Two-view case}
Let there be given two cameras $P_1 = \begin{bmatrix} I & 0 \end{bmatrix}$ and $P_2 = \begin{bmatrix} A & a \end{bmatrix}$, where~$A$ is a $3\times 3$ matrix and~$a$ is a 3-vector. Let~$Q$ be a point in 3-space, and $q_k$ be its $k$th image. Then,
\[
q_k \sim P_k Q, \quad k = 1, 2.
\]
The \emph{incidence relation} for a pair $(q_1, q_2)$ says
\begin{equation}
\label{eq:inci2}
q_2^{\mathrm T} F q_1 = 0,
\end{equation}
where matrix $F = [a]_\times A$ is called the \emph{fundamental matrix}. It is important that relation~\eqref{eq:inci2} is linear in the entries of~$F$, so that given a number of point correspondences (eight or more), one can estimate the entries of~$F$ by solving a linear system.

It follows from the definition of matrix~$F$ that $\det F = 0$. One easily verifies that this condition is also sufficient. Thus we have

\begin{theorem}[\cite{HZ}]
\label{thm:fund}
A real non-zero $3\times 3$ matrix~$F$ is a fundamental matrix if and only if
\begin{equation}
\label{eq:constr1}
\det F = 0.
\end{equation}
\end{theorem}

The \emph{essential matrix}~$E$ is the fundamental matrix for calibrated cameras $\hat P_1 = \begin{bmatrix} I & 0 \end{bmatrix}$ and $\hat P_2 = \begin{bmatrix} R & t\end{bmatrix}$, where $R \in \mathrm{SO}(3)$, $t$ is a 3-vector, that is
\begin{equation}
\label{eq:essential}
E = [t]_\times R.
\end{equation}
The matrices $F$ and $E$ are related by
\begin{equation}
\label{eq:F}
F \sim K_2^{-\mathrm T} E K_1^{-1},
\end{equation}
where $K_k$ is the calibration matrix of the $k$th camera. It follows that the incidence relation~\eqref{eq:inci2} for the essential matrix becomes
\[
\hat q_2^{\mathrm T} E \hat q_1 = 0,
\]
where $\hat q_k = K_k^{-1} q_k$ are the so-called \emph{normalized coordinates}.

Equality~\eqref{eq:essential} can be thought of as the definition of the essential matrix, i.e. it is a $3\times 3$ non-zero skew-symmetric matrix post-multiplied by a special orthogonal matrix. Moreover, we can even consider complex essential matrices assuming that in~\eqref{eq:essential} vector~$t \in \mathbb C^3$ and matrix~$R \in \mathrm{SO}(3, \mathbb C)$.

The real fundamental matrix has 7 degrees of freedom, whereas the real essential matrix has only~5 degrees of freedom. It is translated into the following property~\cite{FM}, \cite{HZ}, \cite{HF89}: two of singular values of matrix~$E$ are equal and the third is zero. The condition is also sufficient. An equivalent form of this result is given by
\begin{theorem}[\cite{Dem88}, \cite{FM}]
\label{thm:essential1}
A real $3\times 3$ matrix~$E$ is an essential matrix if and only if
\begin{align}
\label{eq:constr2_1}
\det E &= 0,\\
\label{eq:constr2_2}
\tr(EE^{\mathrm T})^2 - 2\tr((EE^{\mathrm T})^2) &= 0.
\end{align}
\end{theorem}

We emphasize that constraints~\eqref{eq:constr2_1} and~\eqref{eq:constr2_2} characterize only \emph{real} essential matrices. There exist ``non-essential'' complex $3\times 3$ matrices which nevertheless satisfy both conditions~\eqref{eq:constr2_1} and~\eqref{eq:constr2_2}. The most general form of such matrices will be given in the next section.

The following theorem gives another characterization constraint on the entries of essential matrix~$E$. It is also valid in case of complex~$E$.
\begin{theorem}[\cite{Dem88}, \cite{FM}, \cite{Maybank}]
\label{thm:essential2}
A real or complex $3\times 3$ matrix~$E$ of rank two is an essential matrix if and only if
\begin{equation}
\label{eq:constr2_3}
(\tr(EE^{\mathrm T}) I - 2EE^{\mathrm T}) E = 0_{3\times 3}.
\end{equation}
\end{theorem}

We note that there are ``non-essential'' rank one matrices which satisfy~\eqref{eq:constr2_3}. However, it can be shown that all of them are limits of sequences of essential matrices~\cite{Maybank}. Thus the closure of the set of essential matrices constitutes an algebraic variety generated by~\eqref{eq:constr2_3}.

It is interesting to note that Theorem~\ref{thm:essential2} is a key for developing efficient algorithms of the essential matrix estimation from five point correspondences in two views~\cite{Nister}.

\subsection{Three-view case}

A $(2, 1)$ tensor is a valency 3 tensor with two contravariant and one covariant indices. For a $(2, 1)$ tensor~$T$ we write $T = \begin{bmatrix}T_1 & T_2 & T_3\end{bmatrix}$, where~$T_k$ are $3\times 3$ matrices corresponding to the covariant index.

Let there be given three cameras $P_1 = \begin{bmatrix} I & 0 \end{bmatrix}$, $P_2 = \begin{bmatrix} A & a \end{bmatrix}$ and $P_3 = \begin{bmatrix} B & b \end{bmatrix}$, where~$A$ and~$B$ are $3\times 3$ matrices, $a$ and~$b$ are 3-vectors. The \emph{trifocal tensor}~$T = \begin{bmatrix}T_1 & T_2 & T_3\end{bmatrix}$ is a $(2, 1)$ tensor defined by
\begin{equation}
\label{eq:Tk}
T_k = A e_k b^{\mathrm T} - a e_k^{\mathrm T} B^{\mathrm T},
\end{equation}
where $e_1$, $e_2$, $e_3$ constitute the standard basis in~$\mathbb R^3$. For a trifocal tensor~$T$ matrices~$T_k$ are called the \emph{correlation slices}.

It is clear that $\det T_k = 0$. If matrices~$T_k$ are of rank two, then let~$l_k$ and~$r_k$ be the left and right null vectors of~$T_k$ respectively. It follows from~\eqref{eq:Tk} that $l_k = [a]_\times A e_k$ and $r_k = [b]_\times B e_k$. Therefore the two (sextic in the entries of $T_1, T_2, T_3$) \emph{epipolar constraints} hold~\cite{HZ}, \cite{Faugeras}:
\begin{equation}
\label{eq:epipolar}
\begin{split}
\det \begin{bmatrix}l_1 & l_2 & l_3\end{bmatrix} = \det ([a]_\times A) &= 0,\\
\det \begin{bmatrix}r_1 & r_2 & r_3\end{bmatrix} = \det ([b]_\times B) &= 0.
\end{split}
\end{equation}
Moreover, for any scalars $\alpha, \beta, \gamma$, the matrix $\alpha T_1 + \beta T_2 + \gamma T_3$ is also degenerate (its right null vector is $[b]_\times B (\alpha e_1 + \beta e_2 + \gamma e_3)$) meaning that
\begin{equation}
\label{eq:rank}
\det (\alpha T_1 + \beta T_2 + \gamma T_3) = 0.
\end{equation}
This equality is referred to as the \emph{extended rank constraint}. It is equivalent to ten (cubic in the entries of $T_1, T_2, T_3$) equations each of which is the coefficient in $\alpha^i \beta^j \gamma^k$ with $i + j + k = 3$.

\begin{theorem}[\cite{FL}, \cite{Faugeras}]
\label{thm:Faugeras}
Let $T = \begin{bmatrix}T_1 & T_2 & T_3\end{bmatrix}$ be a real $(2, 1)$ tensor such that $\rank T_k = 2$, $k = 1, 2, 3$. Let~$T$ satisfy the two epipolar~\eqref{eq:epipolar} and ten extended rank~\eqref{eq:rank} constraints. Let the ranks of matrices $\begin{bmatrix}l_1 & l_2 & l_3\end{bmatrix}$ and $\begin{bmatrix}r_1 & r_2 & r_3\end{bmatrix}$ equal two. Then~$T$ is a trifocal tensor.
\end{theorem}

\begin{remark}
The additional rank constraints from Theorem~\ref{thm:Faugeras} are sometimes referred to as the ``general viewpoint assumption''. The following example demonstrates that they can not be omitted. The $(2, 1)$ tensor
\[
T = \begin{bmatrix}\begin{bmatrix}0 & 1 & 0\\0 & 0 & 1\\ 0 & 0 & 0\end{bmatrix} & \begin{bmatrix}1 & 0 & 0\\ 0 & 1 & 0\\ 0 & 0 & 0\end{bmatrix} & \begin{bmatrix}1 & 0 & 0\\ 0 & 1 & 0\\ 0 & 0 & 0\end{bmatrix}\end{bmatrix}
\]
satisfies both the epipolar and extended rank constraints. However it is not a trifocal tensor, i.e. it can not be represented in form~\eqref{eq:Tk}. On the other hand, there exist degenerate trifocal tensors such that at least one of matrices $\begin{bmatrix}l_1 & l_2 & l_3\end{bmatrix}$, $\begin{bmatrix}r_1 & r_2 & r_3\end{bmatrix}$ or even~$T_k$ is of rank less than two.
\end{remark}

Let $q_k$ be the $k$th image of a point~$Q$ in 3-space. The \emph{trifocal incidence relation} for a triple $(q_1, q_2, q_3)$ says~\cite{HZ}
\begin{equation}
\label{eq:inci3}
[q_2]_\times \sum\limits_{j = 1}^3(q_1)_j T_j [q_3]_\times = 0_{3\times 3}.
\end{equation}
It is important that relation~\eqref{eq:inci3} is linear in the entries of~$T$.

The \emph{calibrated trifocal tensor} $\hat T$ is the trifocal tensor for calibrated cameras $P_1 = \begin{bmatrix} I & 0 \end{bmatrix}$, $P_2 = \begin{bmatrix} R_2 & t_2 \end{bmatrix}$ and $P_3 = \begin{bmatrix} R_3 & t_3 \end{bmatrix}$, where $R_2, R_3 \in \mathrm{SO}(3)$, $t_2, t_3 \in \mathbb R^3$, i.e.
\begin{equation}
\label{eq:calibT}
\hat T_k = R_2 e_k t_3^{\mathrm T} - t_2 e_k^{\mathrm T} R_3^{\mathrm T}.
\end{equation}

The calibrated trifocal tensor is an analog of the essential matrix in three views. The tensors~$T$ and~$\hat T$ are related by
\begin{equation}
\label{eq:Tj}
T_j \sim K_2 \sum\limits_{k = 1}^3(K_1^{-\mathrm T})_{jk}\hat T_k K_3^{\mathrm T},
\end{equation}
where $K_k$ is the calibration matrix of the $k$th camera.

For any invertible $3\times 3$ matrix~$M$ and 3-vector~$t$, the following identity holds:
\[
[M^{-1}t]_\times = \det(M^{-1}) M^{\mathrm T}[t]_\times M.
\]
Therefore the trifocal incidence relation~\eqref{eq:inci3} for a calibrated trifocal tensor becomes
\[
[\hat q_2]_\times \sum\limits_{j = 1}^3(\hat q_1)_j \hat T_j [\hat q_3]_\times = 0_{3\times 3},
\]
where $\hat q_k = K_k^{-1} q_k$ are the normalized coordinates.

The tensors $T$ and~$\hat T$ have 18 and 11 degrees of freedom respectively. It follows that matrices~$\hat T_k$ must satisfy a number of additional algebraic constraints. Some of them are described below.

\section{The Trifocal Essential Matrix and Its Characterization}
\label{sec:character}

The \emph{trifocal essential matrix} is, by definition, a $3\times 3$ matrix~$S$ which can be represented in form
\begin{equation}
\label{eq:formS}
S = s_1 t_1^{\mathrm T} + t_2 s_2^{\mathrm T},
\end{equation}
where $t_1, t_2, s_1, s_2 \in \mathbb C^3$, and vectors $s_1, s_2$ are non-zero and such that $s_k^{\mathrm T}s_k = 0$, $k = 1, 2$. It is clear that matrices~$S$, $S^{\mathrm T}$ and $RSQ$, where $R, Q \in \mathrm{SO}(3, \mathbb C)$, simultaneously are (or are not) the trifocal essential matrices.

\begin{lemma}
\label{lem:eigenval}
Let $a, b, c, d \in \mathbb C^n$. Then the (possibly) non-zero eigenvalues of matrix $M = a c^{\mathrm T} + b d^{\mathrm T}$ coincide with the eigenvalues of $2\times 2$ matrix
\[
N =
\begin{bmatrix}
c^{\mathrm T} a & c^{\mathrm T} b \\
d^{\mathrm T} a & d^{\mathrm T} b
\end{bmatrix}.
\]
\end{lemma}

\begin{proof}
The rank of matrix $M$ is at most~2. Let $\lambda_1$, $\lambda_2$ be the (possibly) non-zero eigenvalues of~$M$. Then,
\[
\lambda_1 + \lambda_2 = \tr(M) = c^{\mathrm T} a + d^{\mathrm T} b = \tr(N),
\]
\begin{multline*}
2\lambda_1\lambda_2 = (\lambda_1 + \lambda_2)^2 - (\lambda_1^2 + \lambda_2^2) = \tr(M)^2 - \tr(M^2) \\= 2(c^{\mathrm T}a)(d^{\mathrm T}b) - 2(c^{\mathrm T}b)(d^{\mathrm T}a) = 2\det N.
\end{multline*}
We see that $\lambda_1$, $\lambda_2$ are the eigenvalues of matrix~$N$, as required.
\end{proof}

\begin{theorem}
\label{thm:eigenval}
Let a $3\times 3$ matrix~$S$ be a trifocal essential matrix. Then $SS^{\mathrm T}$ has one zero and two other equal eigenvalues.
\end{theorem}

\begin{proof}
Let $S$ be a trifocal essential matrix, i.e. it can be represented in form~\eqref{eq:formS}. Matrix $SS^{\mathrm T}$ has zero eigenvalue, as $\det S = 0$. Taking into account that $s_2^{\mathrm T}s_2 = 0$, we get
\begin{equation}
\label{eq:matrix1}
SS^{\mathrm T} = s_1(\mu s_1^{\mathrm T} + \nu t_2^{\mathrm T}) + \nu t_2s_1^{\mathrm T},
\end{equation}
where we have denoted $\mu = t_1^{\mathrm T}t_1$, $\nu = s_2^{\mathrm T}t_1$. By Lemma~\ref{lem:eigenval}, the potentially non-zero eigenvalues of~\eqref{eq:matrix1} are equal to the ones of $2\times 2$ matrix
\[
\begin{bmatrix}
\nu t_2^{\mathrm T}s_1 & \nu(\mu s_1^{\mathrm T} + \nu t_2^{\mathrm T}) t_2 \\ 0 & \nu s_1^{\mathrm T}t_2
\end{bmatrix},
\]
and the eigenvalues of the latter matrix are both equal to $\nu s_1^{\mathrm T}t_2 = (s_1^{\mathrm T}t_2)(s_2^{\mathrm T}t_1)$. Theorem~\ref{thm:eigenval} is proved.
\end{proof}

\begin{lemma}
\label{lem:constr3_2}
Let~$M$ be a degenerate $3\times 3$ matrix. Then the two (possibly) non-zero eigenvalues of~$M$ coincide if and only if the entries of~$M$ are subject to
\begin{equation}
\label{eq:constr3_2}
\tr(M)^2 - 2\tr(M^2) = 0.
\end{equation}
\end{lemma}

\begin{proof}
Let $0$, $\lambda_1$, $\lambda_2$ be the eigenvalues of~$M$. Then,
\[
\tr(M)^2 - 2\tr(M^2) = (\lambda_1 + \lambda_2)^2 - 2(\lambda_1^2 + \lambda_2^2) = -(\lambda_1 - \lambda_2)^2.
\]
It follows that $\lambda_1 = \lambda_2$ if and only if~\eqref{eq:constr3_2} holds. Lemma~\ref{lem:constr3_2} is proved.
\end{proof}

\begin{lemma}[\cite{Maybank}]
\label{lem:Maybank}
Let $s_1, s_2 \in \mathbb C^3$ be any non-zero vectors satisfying $s_k^{\mathrm T}s_k = 0$. Then there exists a matrix $R \in \mathrm{SO}(3, \mathbb C)$ such that $Rs_1 = s_2$.
\end{lemma}

\begin{theorem}
\label{thm:suff1}
A $3\times 3$ matrix $S$ is a trifocal essential matrix if and only if
\begin{align}
\label{eq:constr5_1}
\det S &= 0,\\
\label{eq:constr5_2}
\tr(SS^{\mathrm T})^2 - 2\tr((SS^{\mathrm T})^2) &= 0.
\end{align}
\end{theorem}

\begin{proof}
The``only if'' part is due to Theorem~\ref{thm:eigenval} and Lemma~\ref{lem:constr3_2}. To prove the ``if'' part, let~$S$ be a $3\times 3$ matrix satisfying~\eqref{eq:constr5_1}, \eqref{eq:constr5_2}. We denote~$c_k$ the $k$th column of matrix~$S$. Because~$S$ is degenerate, there exists a non-zero vector~$a$ such that $Sa = 0$. There are two possibilities.

\medskip
\noindent\textbf{Case 1: $a^{\mathrm T}a \neq 0$.} Scaling~$a$ and post-multiplying~$S$ by an appropriate matrix from $\mathrm{SO}(3, \mathbb C)$, we assume without loss of generality that $a = \begin{bmatrix}0 & 0 & 1\end{bmatrix}^{\mathrm T}$. Therefore $c_3 = 0$.

Suppose first that either $c_1^{\mathrm T}c_1 \neq 0$ or $c_2^{\mathrm T}c_2 \neq 0$. Without loss of generality we assume $c_2^{\mathrm T}c_2 \neq 0$. Pre-multiplying~$S$ by an appropriate rotation, we obtain
\[
S = \begin{bmatrix}\lambda & \mu & 0\\ \nu & 0 & 0\\ 0 & 0 & 0\end{bmatrix}.
\]
The substitution of~$S$ into~\eqref{eq:constr5_2} gives
\[
((\mu + \nu)^2 + \lambda^2)((\mu - \nu)^2 + \lambda^2) = 0.
\]
It follows that $\lambda = i(\epsilon_1\mu + \epsilon_2\nu)$, where $\epsilon_k = \pm 1$. Thus,
\[
S = \begin{bmatrix}i(\epsilon_1\mu + \epsilon_2\nu) & \mu & 0\\ \nu & 0 & 0\\ 0 & 0 & 0\end{bmatrix} = \begin{bmatrix}i\epsilon_2 \\ 1 \\ 0\end{bmatrix} \begin{bmatrix}\nu & 0 & 0\end{bmatrix} + \begin{bmatrix}\mu \\ 0 \\ 0\end{bmatrix} \begin{bmatrix}i\epsilon_1 & 1 & 0\end{bmatrix}.
\]

Consider the case $c_1^{\mathrm T}c_1 = c_2^{\mathrm T}c_2 = 0$. Due to Lemma~\ref{lem:Maybank}, we can pre-multiply~$S$ by an appropriate rotation to get
\[
S = \begin{bmatrix}\alpha & 1 & 0\\ \beta & i & 0\\ \gamma & 0 & 0\end{bmatrix},
\]
where $\alpha^2 + \beta^2 + \gamma^2 = 0$. The substitution of~$S$ into~\eqref{eq:constr5_2} yields
\[
4(i\alpha - \beta)^2 = 0.
\]
It follows that $\beta = i\alpha$ and $\gamma = 0$. Therefore matrix~$S$ has rank one and
\[
S = \begin{bmatrix}\alpha & 1 & 0\\ i\alpha & i & 0\\ 0 & 0 & 0\end{bmatrix} = \begin{bmatrix}1 \\ i \\ 0\end{bmatrix} \begin{bmatrix}\alpha & 1 & 0\end{bmatrix} + 0 s^{\mathrm T},
\]
where $s$ is an arbitrary 3-vector satisfying $s^{\mathrm T} s = 0$. Thus in either case~$S$ is a trifocal essential matrix, as required.

\medskip
\noindent\textbf{Case 2: $a^{\mathrm T}a = 0$.} Due to Lemma~\ref{lem:Maybank}, we can post-multiply~$S$ by an appropriate matrix from $\mathrm{SO}(3, \mathbb C)$ and suppose without loss of generality that $a = \begin{bmatrix}0 & 1 & i\end{bmatrix}^{\mathrm T}$. Therefore $c_3 = ic_2$.

By direct computation, equality~\eqref{eq:constr5_2} becomes $(c_1^{\mathrm T}c_1)^2 = 0$, i.e. $c_1^{\mathrm T}c_1 = 0$. This yields
\[
S = \begin{bmatrix}\alpha & \lambda & i\lambda\\ \beta & \mu & i\mu\\ \gamma & \nu & i\nu\end{bmatrix} = \begin{bmatrix}\alpha \\ \beta \\ \gamma\end{bmatrix} \begin{bmatrix}1 & 0 & 0\end{bmatrix} + \begin{bmatrix}\lambda \\ \mu \\ \nu\end{bmatrix} \begin{bmatrix}0 & 1 & i\end{bmatrix},
\]
where $\alpha^2 + \beta^2 + \gamma^2 = 0$, i.e. $S$ is a trifocal essential matrix. Theorem~\ref{thm:suff1} is proved.
\end{proof}

We notice that constraints~\eqref{eq:constr5_1},~\eqref{eq:constr5_2} coincide with constraints~\eqref{eq:constr2_1},~\eqref{eq:constr2_2} from Theorem~\ref{thm:essential1}. Hence, if a trifocal essential matrix is real, then it is an essential matrix.

In general, a trifocal essential matrix does not satisfy cubic constraint~\eqref{eq:constr2_3}. The proof consists in exhibiting a counterexample. Let $s_1 = s_2 = \begin{bmatrix}1 & i & 0\end{bmatrix}^{\mathrm T}$, $t_1 = t_2 = \begin{bmatrix}1 & 0 & 0\end{bmatrix}^{\mathrm T}$. Then $S = \begin{bmatrix}2 & i & 0 \\ i & 0 & 0\\ 0 & 0 & 0\end{bmatrix}$ and the eigenvalues of $SS^{\mathrm T}$ are $0, 1, 1$. However,
\[
(\tr(SS^{\mathrm T})I - 2SS^{\mathrm T})S = -4\begin{bmatrix}1 & i & 0 \\ i & -1 & 0\\ 0 & 0 & 0\end{bmatrix} \neq 0_{3\times 3}.
\]

Nevertheless, there exists an analog of identity~\eqref{eq:constr2_3} for trifocal essential matrices.

\begin{theorem}
\label{thm:crit1}
A $3\times 3$ matrix $S$ is a trifocal essential matrix if and only if
\begin{equation}
\label{eq:constr6}
(\tr(SS^{\mathrm T}) I - 2SS^{\mathrm T})^2 S = 0_{3\times 3}.
\end{equation}
\end{theorem}

\begin{proof}
Let us denote
\[
\Phi(M) = (\tr(MM^{\mathrm T}) I - 2MM^{\mathrm T})^2 M
\]
and
\[
\varphi(M) = \tr(MM^{\mathrm T})^2 - 2\tr((MM^{\mathrm T})^2).
\]
Then it is straightforward to show that for arbitrary $3\times 3$ matrix~$M$ the following identity holds:
\begin{equation}
\label{PhiM}
\Phi(M) = 4M^*\det M - M\varphi(M),
\end{equation}
where $M^*$ is meant the matrix of cofactors of~$M$.

Let matrix~$S$ be a trifocal essential matrix. By Theorem~\ref{thm:suff1}, $\det S = \varphi(S) = 0$. Then it follows from~\eqref{PhiM} that $\Phi(S) = 0_{3\times 3}$, i.e.~\eqref{eq:constr6} holds. The ``only if'' part is proved.

Conversely, let a $3\times 3$ matrix~$S$ satisfy~\eqref{eq:constr6}, i.e. $\Phi(S) = 0_{3\times 3}$. It suffices to show that $\det S = 0$. Suppose, by hypothesis, that $\det S \neq 0$. Then, post-multiplying~\eqref{eq:constr6} by~$S^{-1}$, we get
\[
(\tr(SS^{\mathrm T}) I - 2SS^{\mathrm T})^2 = 0_{3\times 3}.
\]
It follows that all the eigenvalues of~$\tr(SS^{\mathrm T}) I - 2SS^{\mathrm T}$ are zeroes and
\[
\tr(\tr(SS^{\mathrm T}) I - 2SS^{\mathrm T}) = \tr(SS^{\mathrm T}) = 0.
\]
The substitution of this into~\eqref{eq:constr6} yields $(\det S)^5 = 0$ in contradiction to the hypothesis $\det S \neq 0$. Thus, $\det S = 0$ and, by~\eqref{PhiM}, $\varphi(S) = 0$. By Theorem~\ref{thm:suff1}, matrix~$S$ is a trifocal essential matrix. Theorem~\ref{thm:crit1} is proved.
\end{proof}

To summarize, the above theorems imply the following statements.
\begin{itemize}
\item The pair of scalar constraints~\eqref{eq:constr5_1},~\eqref{eq:constr5_2} is equivalent to the single matrix constraint~\eqref{eq:constr6}.
\item The most general form of a $3\times 3$ matrix satisfying equations~\eqref{eq:constr5_1} and~\eqref{eq:constr5_2} is the trifocal essential matrix given by~\eqref{eq:formS}.
\item If a trifocal essential matrix is real, then it is an essential matrix.
\item Every essential matrix is a trifocal essential matrix, but the converse is not true in general.
\end{itemize}

\section{Three Necessary Conditions on a Calibrated Trifocal Tensor}
\label{sec:conds}

A new notion of trifocal essential matrix, introduced in the previous section, turns out to be closely related to calibrated trifocal tensors. The connection is established by the following lemma.

\begin{lemma}
\label{lem:hatS}
Let $\hat T = \begin{bmatrix}\hat T_1 & \hat T_2 & \hat T_3\end{bmatrix}$ be a calibrated trifocal tensor. Then a $3\times 3$ matrix $S_{\hat T} = \alpha \hat T_1 + \beta \hat T_2 + \gamma \hat T_3$, where numbers $\alpha, \beta, \gamma$ are such that $\alpha^2 + \beta^2 + \gamma^2 = 0$, is a trifocal essential matrix, i.e. it can be represented in form~\eqref{eq:formS}.
\end{lemma}

\begin{proof}
We notice that
\[
S_{\hat T} = \alpha \hat T_1 + \beta \hat T_2 + \gamma \hat T_3 = R_2 s t_3^{\mathrm T} - t_2 s^{\mathrm T} R_3^{\mathrm T} = s_2 t_3^{\mathrm T} + (-t_2) s_3^{\mathrm T},
\]
where $s = \begin{bmatrix}\alpha & \beta & \gamma\end{bmatrix}^{\mathrm T}$, and $s_k = R_k s$ are 3-vectors satisfying
\[
s_k^{\mathrm T} s_k = s^{\mathrm T}R_k^{\mathrm T}R_k s = s^{\mathrm T}s = 0.
\]
It follows that~$S_{\hat T}$ is a trifocal essential matrix. Lemma~\ref{lem:hatS} is proved.
\end{proof}

We call matrix $S_{\hat T}(s) = \alpha \hat T_1 + \beta \hat T_2 + \gamma \hat T_3$ the \emph{trifocal essential matrix associated with $\hat T = \begin{bmatrix}\hat T_1 & \hat T_2 & \hat T_3\end{bmatrix}$}. It has the following geometric interpretation.

A conic $\Omega_\infty \subset \mathbb P^3$ consisting of points $\begin{bmatrix}\alpha & \beta & \gamma & 0\end{bmatrix}^{\mathrm T}$ with $\alpha^2 + \beta^2 + \gamma^2 = 0$ is known as the \emph{absolute conic}~\cite{HZ}. It lies on the plane at infinity and does not have any real points.

Let the camera matrices be $P_1 = \begin{bmatrix} K_1 & 0 \end{bmatrix}$, $P_2 = K_2\begin{bmatrix} R_2 & t_2 \end{bmatrix}$, $P_3 = K_3\begin{bmatrix} R_3 & t_3 \end{bmatrix}$, where~$K_k$ is the calibration matrix of the $k$th camera. Then the $k$th image of~$\Omega_\infty$ is $\omega_k = (K_kK_k^{\mathrm T})^{-1}$. Let $s = \begin{bmatrix}\alpha & \beta & \gamma\end{bmatrix}^{\mathrm T}$, $Q = \begin{bmatrix}s & 0\end{bmatrix}^{\mathrm T} \in \Omega_\infty$ and $q_1 \sim P_1 Q = K_1 s$, $q_k \sim P_k Q = K_k R_k s$, $k = 2, 3$. Then, by~\eqref{eq:Tj},
\[
S_T(q_1) = \sum\limits_{j = 1}^3(q_1)_jT_j \sim K_2 S_{\hat T}(s) K_3^{\mathrm T}.
\]
It follows that
\begin{equation}
\label{eq:ST1}
S_T(q_1) [q_3]_{\times}p_3 \sim K_2 S_{\hat T}(s) [R_3s]_{\times} K_3^{-1} p_3 \sim K_2 R_2 s \sim q_2
\end{equation}
for arbitrary point $p_3$ in the third image. Thus, the rank deficient matrix~$S_T(q_1)$ represents a mapping $\mathbb P^1 \to \mathbb P^0$ from the pencil of lines through the point~$q_3 \in \omega_3$ in the third image to the corresponding point~$q_2 \in \omega_2$ in the second image (Fig.~\ref{fig:trifessen}).

\begin{figure}[t]
\centering
\includegraphics[width=0.6\hsize]{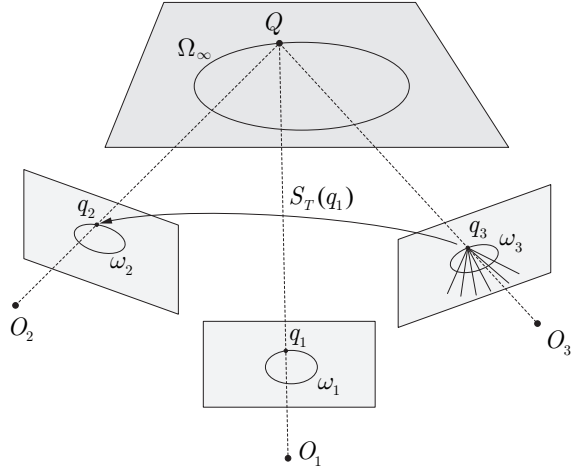}
\caption{Geometric interpretation of the trifocal essential matrix associated with a calibrated trifocal tensor}\label{fig:trifessen}
\end{figure}

\bigskip

In the rest of this section calibrated trifocal tensors are allowed to have complex entries, that is in~\eqref{eq:calibT} matrices $R_2, R_3$ belong to~$\mathrm{SO}(3, \mathbb C)$, and vectors $t_2, t_3$ are in~$\mathbb C^3$.

Let us introduce six symmetric matrices ($k = 1, 2, 3$)
\begin{equation}
\label{eq:UkVk}
\begin{split}
U_k &= \hat T_k \hat T_k^{\mathrm T},\\
V_k &= \hat T_k \hat T_{k + 1}^{\mathrm T} + \hat T_{k + 1} \hat T_k^{\mathrm T}.
\end{split}
\end{equation}
Here $k + 1$ should be read as $k \pmod 3 +1$, i.e. $V_3 = \hat T_3 \hat T_1^{\mathrm T} + \hat T_1 \hat T_3^{\mathrm T}$.

\begin{theorem}[1st necessary condition]
\label{thm:nine}
Let $\hat T = \begin{bmatrix}\hat T_1 & \hat T_2 & \hat T_3\end{bmatrix}$ be a calibrated trifocal tensor, matrices~$U_k$,~$V_k$ be defined in~\eqref{eq:UkVk}. Then the entries of~$\hat T_1$, $\hat T_2$, $\hat T_3$ are constrained by the following equations:
\begin{align}
\label{eq:quart1}
\psi(U_3 - U_1, U_3 - U_1) - \psi(V_3, V_3) &= 0,\\
\label{eq:quart2}
\psi(U_3 - U_1, V_1) + \psi(V_2, V_3) &= 0,\\
\label{eq:quart3}
\psi(U_1 - U_2, V_1) &= 0,
\end{align}
where $\psi(X, Y) = \tr(X)\tr(Y) - 2\tr(XY)$. Six more equations are obtained from~\eqref{eq:quart1} -- \eqref{eq:quart3} by a cyclic permutation of indices $1 \to 2 \to 3 \to 1$. The resulting nine equations are linearly independent.
\end{theorem}

\begin{proof}
Let $S_{\hat T} = \alpha \hat T_1 + \beta \hat T_2 + \gamma \hat T_3$ be a trifocal essential matrix associated with~$\hat T$. By Theorem~\ref{thm:suff1}, the following equation holds:
\begin{equation}
\label{eq:constr3_3}
\tr(S_{\hat T}S_{\hat T}^{\mathrm T})^2 - 2\tr((S_{\hat T}S_{\hat T}^{\mathrm T})^2) = 0.
\end{equation}

The definition of matrices~$U_k$,~$V_k$ (see~\eqref{eq:UkVk}) permits us to write
\[
S_{\hat T}S_{\hat T}^{\mathrm T} = \alpha^2 U_1 + \beta^2 U_2 + \gamma^2 U_3 + \alpha\beta V_1 + \beta\gamma V_2 + \gamma\alpha V_3.
\]
Substituting this into~\eqref{eq:constr3_3}, we find the coefficients in~$\alpha^4$, $\alpha^3\beta$ and $\alpha\beta^3$ taking into account that $\gamma^2 = -\alpha^2 - \beta^2$. Because $\alpha$ and~$\beta$ are arbitrary, these coefficients must vanish:
\begin{align}
\label{eq:coeff1}
\alpha^4 &\colon \psi(U_3 - U_1, U_3 - U_1) - \psi(V_3, V_3) = 0,\\
\label{eq:coeff2}
\alpha^3\beta &\colon \psi(U_1 - U_3, V_1) - \psi(V_2, V_3) = 0,\\
\label{eq:coeff3}
\alpha\beta^3 &\colon \psi(U_2 - U_3, V_1) - \psi(V_2, V_3) = 0.
\end{align}
Thus we get~$\eqref{eq:quart1} = \eqref{eq:coeff1}$,~$\eqref{eq:quart2} = -\eqref{eq:coeff2}$, and $\eqref{eq:quart3} = \eqref{eq:coeff2} - \eqref{eq:coeff3}$. It is clear that we can get six more constraints on~$\hat T_k$ from~\eqref{eq:quart1}~--~\eqref{eq:quart3} by a cyclic permutation of the indices.

Finally, the resulting nine polynomials can not be linearly dependent, since each of them contains monomials that are not contained in all the other polynomials. Examples of such monomials for~\eqref{eq:quart1}, \eqref{eq:quart2} and~\eqref{eq:quart3} are $(\hat T_3)^2_{11}(\hat T_1)^2_{11}$, $(\hat T_3)^2_{11}(\hat T_1)_{11}(\hat T_2)_{11}$ and~$(\hat T_1)^3_{11}(\hat T_2)_{11}$ respectively. Theorem~\ref{thm:nine} is proved.
\end{proof}

From now on the nine equalities from Theorem~\ref{thm:nine} will be referred to as the \emph{eigenvalue constraints}.

\begin{theorem}
\label{thm:nine2}
Let $T = \begin{bmatrix}T_1 & T_2 & T_3\end{bmatrix}$ be a $(2, 1)$ tensor satisfying the ten extended rank and nine eigenvalue constrains. Then matrix $S_T = \alpha T_1 + \beta T_2 + \gamma T_3$ with $\alpha^2 + \beta^2 + \gamma^2 = 0$ is a trifocal essential matrix.
\end{theorem}

\begin{proof}
The extended rank constraints imply $\det S_T = 0$. Taking into account that $\gamma^2 = -\alpha^2 - \beta^2$, we conclude that the expression
\[
\varphi(S_T) = \tr(S_TS_T^{\mathrm T})^2 - 2\tr((S_TS_T^{\mathrm T})^2)
\]
contains 9 monomials:
\[
\alpha^4, \alpha^3\beta, \alpha^3\gamma, \alpha^2\beta^2, \alpha^2\beta\gamma, \alpha\beta^3, \alpha\beta^2\gamma, \beta^4, \beta^3\gamma.
\]
It is directly verified that the coefficients in all of them are linear combinations of the nine polynomials from Theorem~\ref{thm:nine}, i.e. $\varphi(S_T) = 0$. By Theorem~\ref{thm:suff1}, $S_T$ is a trifocal essential matrix, as required.
\end{proof}

\begin{theorem}[2nd necessary condition]
\label{thm:99}
Let $\hat T = \begin{bmatrix}\hat T_1 & \hat T_2 & \hat T_3\end{bmatrix}$ be a calibrated trifocal tensor. Then the entries of~$\hat T_1$, $\hat T_2$, $\hat T_3$ are constrained by the 99 linearly independent quintic (of degree~5) polynomial equations.
\end{theorem}

\begin{proof}
Let $S_{\hat T} = \alpha \hat T_1 + \beta \hat T_2 + \gamma \hat T_3$ be a trifocal essential matrix associated with~$\hat T$. By Theorem~\ref{thm:crit1}, the following equation holds:
\begin{equation}
\label{eq:constr3_4}
(\tr(S_{\hat T}S_{\hat T}^{\mathrm T}) I - 2S_{\hat T}S_{\hat T}^{\mathrm T})^2 S_{\hat T} = 0_{3\times 3}.
\end{equation}
We notice that equality~\eqref{eq:constr3_4} is quintic in the entries of matrix~$S_{\hat T}$. Taking into account that $\gamma^2 = -\alpha^2 - \beta^2$, every of the 9 entries in the l.h.s. of~\eqref{eq:constr3_4} contains 11 monomials in variables~$\alpha$, $\beta$ and~$\gamma$. The coefficient in each of these monomials must vanish. Hence there are in total 99 quintic polynomial constraints on the entries of~$\hat T$. Theorem~\ref{thm:99} is proved.
\end{proof}

\begin{remark}
An explicit form of the quintic polynomial equations from Theorem~\ref{thm:99} is as follows:
\begin{align}
\label{eq:quint1}
%\alpha^5
&(\Psi_1(U_{13}) - \Psi_1(V_3))\hat T_1 - \Psi_2(U_{13}, V_3)\hat T_3 = 0_{3\times 3},\\
\label{eq:quint2}
%\alpha^4\gamma
&\Psi_2(U_{13}, V_3)\hat T_1 + (\Psi_1(U_{13}) - \Psi_1(V_3))\hat T_3 = 0_{3\times 3},\\
\label{eq:quint3}
%\alpha^3\beta\gamma
&(\Psi_2(U_{13}, V_2) + \Psi_2(V_1, V_3))\hat T_1 + \Psi_2(U_{13}, V_3)\hat T_2\notag \\
&\hspace{7\baselineskip}+ (\Psi_2(U_{13}, V_1) - \Psi_2(V_2, V_3))\hat T_3 = 0_{3\times 3},\\
\label{eq:quint4}
%\alpha^4\beta
&(\Psi_2(U_{13}, V_1) - \Psi_2(V_2, V_3))\hat T_1 + (\Psi_1(U_{13}) - \Psi_1(V_3))\hat T_2\notag \\
&\hspace{7\baselineskip}- (\Psi_2(U_{13}, V_2) + \Psi_2(V_1, V_3))\hat T_3 = 0_{3\times 3},
\end{align}
where matrices~$U_k$,~$V_k$ are defined in~\eqref{eq:UkVk}, $U_{jk} = U_j - U_k$, and
\begin{align*}
\Psi(X, Y) &= (\tr(X)I - 2X)(\tr(Y)I - 2Y),\\
\Psi_1(X) &= \Psi(X, X),\\
\Psi_2(X, Y) &= \Psi(X, Y) + \Psi(Y, X).
\end{align*}
Equations~\eqref{eq:quint1}~--~\eqref{eq:quint4} give $4\times 9 = 36$ constraints on~$\hat T_k$. We get $8\times 9 = 72$ more constraints from~\eqref{eq:quint1}~--~\eqref{eq:quint4} by a cyclic permutation of indices $1 \to 2 \to 3 \to 1$. Thus, in total, we have 108 quintic constraints. Let~$M_k$ denote the l.h.s. of the $k$th version of equality~\eqref{eq:quint3}. Then we have
\[
M_1 + M_2 + M_3 \equiv 0_{3\times 3}.
\]
It follows that~\eqref{eq:quint1}~--~\eqref{eq:quint4} give only 99 constraints. Their linear independence is verified directly.
\end{remark}

\begin{remark}
We notice that the 99 quintic constraints from Theorem~\ref{thm:99} are algebraically dependent with the ten extended rank and nine eigenvalue constrains. An explicit form of that dependence is induced by formula~\eqref{PhiM}.
\end{remark}

Finally, we propose the third necessary condition on a calibrated trifocal tensor. It seems not to be directly related to the matrix~$S_{\hat T}$. However this condition could be useful in applications, since it consists of another set of quartic polynomial equations that are satisfied by a calibrated trifocal tensor.

\begin{theorem}[3rd necessary condition]
\label{thm:six}
Let $\hat T = \begin{bmatrix}\hat T_1 & \hat T_2 & \hat T_3\end{bmatrix}$ be a calibrated trifocal tensor. Then the entries of~$\hat T_1$, $\hat T_2$, $\hat T_3$ satisfy the following equations:
\begin{align}
\label{eq:quart4}
\tr(U_2)^2 - \tr(V_3)^2 - \tr(U_2^2 - V_3^2 + (U_3 - U_1)^2) &= 0,\\
\label{eq:quart5}
\tr(V_2)\tr(U_1 - 2U_2 - U_3) - \tr(V_1)\tr(V_3) + 2\tr(V_2 U_2) &= 0,
\end{align}
where matrices~$U_k$,~$V_k$ are defined in~\eqref{eq:UkVk}. Four more equations are obtained from~\eqref{eq:quart4}~--~\eqref{eq:quart5} by a cyclic permutation of indices $1 \to 2 \to 3 \to 1$. The resulting six equations are linearly independent.
\end{theorem}

\begin{proof}
Let tensor $\hat T$ be represented in form~\eqref{eq:calibT}. First we replace~$\hat T$ with $\hat T' = \begin{bmatrix}R_2^{\mathrm T}\hat T_1R_3 & R_2^{\mathrm T}\hat T_2R_3 & R_2^{\mathrm T}\hat T_3R_3\end{bmatrix}$. Then the correlation slices of~$\hat T'$ are simplified to $\hat T'_k = e_kt_3^{\mathrm T} - t_2e_k^{\mathrm T}$. A straightforward computation proves that~$\hat T'$ satisfies equations~\eqref{eq:quart4}~--~\eqref{eq:quart5} and the four their consequences. Then so does~$\hat T$, since the matrices~$U_k$,~$V_k$ are the same for~$\hat T$ and~$\hat T'$.

The resulting six polynomials can not be linearly dependent, since each of them contains monomials that are not contained in all the other polynomials. Examples of such monomials for~\eqref{eq:quart4} and~\eqref{eq:quart5} are $(\hat T_3)^2_{11}(\hat T_1)^2_{11}$ and~$(\hat T_2)_{11}(\hat T_3)^3_{11}$ respectively. Theorem~\ref{thm:six} is proved.
\end{proof}

The 15 equalities from Theorems~\ref{thm:nine} and~\ref{thm:six} will be further referred to as the \emph{quartic constraints}.

\begin{remark}
The eigenvalue constraints do not imply the six equalities from Theorem~\ref{thm:six}. The following trifocal tensor gives a counterexample:
\[
T = \begin{bmatrix}\begin{bmatrix}0 & 0 & 0\\0 & 0 & 1\\ 0 & -1 & 0\end{bmatrix} & \begin{bmatrix}0 & 0 & 1\\ 0 & 0 & 0\\ -1 & 0 & 0\end{bmatrix} & \begin{bmatrix}0 & 0 & 1\\0 & 0 & 1\\ -1 & -1 & 0\end{bmatrix}\end{bmatrix}.
\]
One verifies that~$T$ satisfies the eigenvalue constraints, but not the six constraints from Theorem~\ref{thm:six}.
\end{remark}

\begin{remark}
The quartic constraints are insufficient for a trifocal tensor~$T$ to be calibrated. Here is a counterexample. Consider a $(2, 1)$ tensor
\begin{equation}
\label{eq:cntr}
T = \begin{bmatrix}\begin{bmatrix}i & 0 & 0\\0 & i & 0\\ 0 & 0 & 0\end{bmatrix} & \begin{bmatrix}0 & 0 & i\\-i & -1 & 1\\ 0 & 0 & 0\end{bmatrix} & \begin{bmatrix}1 & 0 & 0\\ -i & 0 & 0\\ i & 1 & 0\end{bmatrix}\end{bmatrix}.
\end{equation}
It is a trifocal tensor, as
\[
T_k = \begin{bmatrix}1 & 0 & 0\\0 & -1 & 0\\ 0 & 0 & 1\end{bmatrix} e_k \begin{bmatrix}i & 1 & 0\end{bmatrix} - \begin{bmatrix}i \\ 1 \\ 0\end{bmatrix} e_k^{\mathrm T} \begin{bmatrix}0 & -i & 0\\0 & 0 & -1\\ i & 0 & 0\end{bmatrix}.
\]
Moreover, $T$ satisfies all the 15 quartic constraints. Suppose that~$T$ is calibrated. Then there must exist 3-vectors $u_k, v_k, t_2, t_3$ such that $u_k^{\mathrm T} u_k = v_k^{\mathrm T} v_k = 1$ and
\[
T_k = u_k t_3^{\mathrm T} - t_2 v_k^{\mathrm T}.
\]
Let us define an ideal:
\[
J = \langle T_k - u_k t_3^{\mathrm T} + t_2 v_k^{\mathrm T}, u_k^{\mathrm T} u_k - 1, v_k^{\mathrm T} v_k - 1 \mid k = 1, 2, 3\rangle \subset \mathbb C[\xi_1, \ldots, \xi_{24}],
\]
where $\xi_j$ are the entries of vectors $t_2, t_3, u_1, u_2, u_3, v_1, v_2$ and~$v_3$. It is straightforward to show by the computation of the Gr\"{o}bner basis of~$J$ that $1 \in J$, and thus~$T$ can not be calibrated.
\end{remark}

\section{A Characterization of Real Calibrated Trifocal Tensors}
\label{sec:suff}

In this section we are going to obtain a three-view analog of condition~\eqref{eq:constr2_2} in Theorem~\ref{thm:essential1}. Namely, we will show that a real trifocal tensor is calibrated if and only if it satisfies the 15 quartic constraints.

First we prove several lemmas.

\begin{lemma}
\label{lem:transform}
Let $T = \begin{bmatrix}T_1 & T_2 & T_3\end{bmatrix}$ be a $(2, 1)$ tensor and $T' = \begin{bmatrix}T'_1 & T'_2 & T'_3\end{bmatrix}$ be a tensor defined by
\begin{equation}
\label{eq:transform}
T'_j = Q_2 \sum\limits_{k = 1}^3(Q_1)_{jk}T_k Q_3^{\mathrm T},
\end{equation}
where $Q_1, Q_2, Q_3 \in \mathrm{SO}(3, \mathbb C)$. Then~$T$ and~$T'$ simultaneously
\begin{enumerate}
\item are (or are not) calibrated trifocal tensors;
\item satisfy (or do not satisfy) the 10 extended rank and 15 quartic constraints.
\end{enumerate}
\end{lemma}

\begin{proof}~

\noindent (1) Let $T$ be a calibrated trifocal tensor, so that its correlation slices can be represented in form~\eqref{eq:calibT}. Then,
\[
T'_j = Q_2 \sum\limits_{k = 1}^3(Q_1)_{jk}T_k Q_3^{\mathrm T} \\= (Q_2R_2Q_1^{\mathrm T}) e_j (Q_3t_3)^{\mathrm T} - (Q_2t_2) e_j^{\mathrm T} (Q_3R_3Q_1^{\mathrm T})^{\mathrm T},
\]
i.e. $T'$ is a calibrated trifocal tensor as well. On the other hand, if~$T$ is not a calibrated trifocal tensor, then so is not~$T'$, since
\begin{equation}
\label{eq:invtrans}
T_k = Q_2^{\mathrm T} \sum\limits_{j = 1}^3(Q_1)_{jk}T'_j Q_3.
\end{equation}

\noindent (2) Let~$T$ be a $(2, 1)$ tensor satisfying the 10 extended rank and 15 quartic constraints. Let us construct matrix $S_T(s) = \sum\limits_{k = 1}^3(s)_k T_k$, where~$s$ is an arbitrary 3-vector. The ten extended rank constraints are then equivalent to $\det S_T(s) = 0$. We get
\begin{equation}
\label{eq:STs}
S_{T'}(s) = \sum\limits_{j = 1}^3(s)_j T'_j = Q_2S_T(Q_1^{\mathrm T}s) Q_3^{\mathrm T},
\end{equation}
and thus $\det S_{T'}(s) = 0$, i.e. tensor~$T'$ satisfies the ten extended rank constraints as well.

Further, if vector~$s$ is such that $s^{\mathrm T}s = 0$, then, by Theorem~\ref{thm:nine2}, matrix~$S_T(s)$ is a trifocal essential matrix. It follows from~\eqref{eq:STs} that~$S_{T'}(s)$ is a trifocal essential matrix too. After that, using the same arguments as in the proof of Theorem~\ref{thm:nine}, one shows that tensor~$T'$ satisfies the nine eigenvalue constraints.

It remains to show that~$T'$ satisfies also the six constraints from Theorem~\ref{thm:six}. We denote by $p_k(T)$ the l.h.s. of the $k$th quartic equation on tensor~$T$ so that $p_{10}, \ldots, p_{15}$ are the six polynomials from Theorem~\ref{thm:six}. Then, by a straightforward computation, we get
\begin{equation}
\label{eq:muj}
p_j(T') = \sum\limits_{k = 1}^{15} \xi_{jk}\, p_k(T), \qquad j = 10, \ldots, 15,
\end{equation}
where $\xi_{jk}$ are polynomial expressions depending only on the entries of matrix~$Q_1$. It follows that if $p_k(T) = 0$ for all~$k$, then also $p_j(T') = 0$ for all~$j$.

Finally, if~$T$ does not satisfy the extended rank and quartic equations, then due to~\eqref{eq:invtrans} so does not~$T'$. This completes the proof of Lemma~\ref{lem:transform}.
\end{proof}

\begin{lemma}
\label{lem:Tsimpl}
Let $T = \begin{bmatrix}T_1 & T_2 & T_3\end{bmatrix}$ be a real trifocal tensor. Then there exist matrices $Q_1, Q_2, Q_3 \in \mathrm{SO}(3)$ such that~$T$ can be transformed by~\eqref{eq:transform} to the trifocal tensor
\begin{equation}
\label{eq:Tsimpl}
T' = \begin{bmatrix}\begin{bmatrix}0 & 0 & \lambda_1\\0 & 0 & 0\\ \nu_1 & \rho_1 & \sigma_1\end{bmatrix} & \begin{bmatrix}0 & 0 & 0\\0 & 0 & \mu_2\\ \nu_2 & \rho_2 & \sigma_2\end{bmatrix} & \begin{bmatrix}0 & 0 & 0\\0 & 0 & 0\\ 0 & \rho_3 & \sigma_3\end{bmatrix}\end{bmatrix},
\end{equation}
where $\lambda_1$, $\mu_2$, $\nu_1$, $\nu_2$, $\rho_k$, $\sigma_k$ are real scalars.
\end{lemma}

\begin{proof}
We are going to explicitly construct rotations $Q_1$, $Q_2$ and~$Q_3$ such that~$T$ is transformed to~$T'$ by~\eqref{eq:transform}. Since~$T$ is a trifocal tensor, we have
\[
T_k = A_2 e_k a_3^{\mathrm T} - a_2 e_k^{\mathrm T} A_3^{\mathrm T}.
\]
Let $H_k \in \mathrm{SO}(3)$ be the Householder matrix such that $H_k a_k = \begin{bmatrix}0 & 0 & \gamma_k\end{bmatrix}^{\mathrm T}$, $k = 2, 3$. First we pre- and post-multiply each $T_k$ by $H_2$ and $H_3^{\mathrm T}$ respectively and denote by~$B_k$ a $2\times 3$ matrix consisting of the first two rows of~$H_kA_k$. After that, we make the singular value decomposition of~$\gamma_3B_2$ to decompose it in form
\[
\gamma_3B_2 = U \begin{bmatrix}\lambda_1 & 0 & 0\\0 & \mu_2 & 0\end{bmatrix} V^{\mathrm T},
\]
where $U \in \mathrm{SO}(2)$ and $V \in \mathrm{SO}(3)$. Finally, let $W \in \mathrm{SO}(2)$ be a rotation such that $(WB_3V)_{13} = 0$. We set
\begin{equation}
\label{eq:Q123}
Q_1 = V^{\mathrm T}, \quad Q_2 = \begin{bmatrix}U^{\mathrm T} & 0\\0 & 1\end{bmatrix}H_2, \quad Q_3 = \begin{bmatrix}W & 0\\0 & 1\end{bmatrix} H_3.
\end{equation}
One verifies that the trifocal tensor~$T$ is transformed to~\eqref{eq:Tsimpl} by the rotations $Q_1$, $Q_2$ and~$Q_3$ defined in~\eqref{eq:Q123}. Lemma~\ref{lem:Tsimpl} is proved.
\end{proof}

We denote by $p_1, \ldots, p_{15}$ the 15 quartic polynomials on the trifocal tensor~$T'$ defined by~\eqref{eq:Tsimpl} and consider an ideal
\begin{equation}
\label{eq:ideal}
J = \langle p_1, \ldots, p_{15} \rangle \subset \mathbb C[\lambda_1, \nu_1, \rho_1, \sigma_1, \mu_2, \nu_2, \rho_2, \sigma_2, \rho_3, \sigma_3].
\end{equation}
Let $\sqrt{J}$ be the radical of~$J$. The following lemma gives a convenient tool to check whether a given polynomial is in the radical of an ideal or not.

\begin{lemma}[\cite{CLS}]
\label{CLS}
Let $J = \langle p_1, \ldots, p_s \rangle \subset \mathbb C[\xi_1, \ldots, \xi_n]$ be an ideal. Then a polynomial $p \in \sqrt{J}$ if and only if $1 \in \tilde J = \langle p_1, \ldots, p_s, 1 - \tau p \rangle \subset \mathbb C[\xi_1, \ldots, \xi_n, \tau]$.
\end{lemma}

We are going to obtain several polynomials that belong to~$\sqrt{J}$. For convenience we divide these polynomials into two parts which are presented in Lemmas~\ref{lem:sqrtJ} and~\ref{lem:sqrtJ1}.

\begin{lemma}
\label{lem:sqrtJ}
The polynomials
\begin{align}
\label{eq:sqrtJ1}
&(\lambda_1^2 - \mu_2^2) (\lambda_1^2 + \sigma_1^2),\\
&(\lambda_1^2 - \mu_2^2) (\mu_2^2 + \sigma_2^2),\\
&\rho_3 (\nu_1\sigma_1 + \nu_2\sigma_2),\\
&\rho_3 (\nu_1\rho_1 + \nu_2\rho_2),\\
&\rho_3 (\rho_3^2 + \sigma_3^2) (\nu_1^2 + \nu_2^2 - \rho_1^2 - \rho_2^2 - \rho_3^2),\\
&(\rho_3^2 + \sigma_3^2) (\rho_1\sigma_1 + \rho_2\sigma_2 + \rho_3(\sigma_3 + \mu_2))(\rho_1\sigma_1 + \rho_2\sigma_2 + \rho_3(\sigma_3 - \mu_2)),\\
\label{eq:sqrtJ5}
&\rho_3(\rho_3^2 + \sigma_3^2) (\nu_1^2 + \nu_2^2 - \sigma_1^2 - \sigma_2^2 - (\sigma_3 + \mu_2)^2)\notag\\
&\hspace{8\baselineskip}\times(\nu_1^2 + \nu_2^2 - \sigma_1^2 - \sigma_2^2 - (\sigma_3 - \mu_2)^2)
\end{align}
belong to~$\sqrt{J}$, where~$J$ is defined in~\eqref{eq:ideal}.
\end{lemma}

\begin{proof}
Let $p$ be any polynomial from~\eqref{eq:sqrtJ1}~--~\eqref{eq:sqrtJ5}. We construct an ideal $\tilde J = J + \langle 1 - \tau p \rangle \subset \mathbb C[\lambda_1, \ldots, \sigma_3, \tau]$, where~$\tau$ is a new variable. By direct computation of the Gr\"{o}bner basis of~$\tilde J$, we get $1 \in \tilde J$. Hence, by Lemma~\ref{CLS}, $p \in \sqrt{J}$. Lemma~\ref{lem:sqrtJ} is proved.
\end{proof}

\begin{remark}
Surprisingly, the computation of the Gr\"{o}bner basis of each~$\tilde J$ takes only a few seconds in Maple even over the field of rationals. In our computations we used the \emph{graded reverse lexicographic order}~\cite{CLS}:
\[
\lambda_1 > \nu_1 > \rho_1 > \sigma_1 > \mu_2 > \nu_2 > \rho_2 > \sigma_2 > \rho_3 > \sigma_3 > \tau.
\]
\end{remark}

\begin{lemma}
\label{lem:sqrtJ1}
The polynomials
\begin{align}
\label{eq:sqrtJ11}
&(\nu_1^2 + \rho_1^2 + \sigma_1^2)(\sigma_1^2 + \sigma_2^2 - \rho_3^2 + \lambda_1^2 - \mu_2^2),\\
&(\nu_2^2 + \rho_2^2 + \sigma_2^2)(\sigma_1^2 + \sigma_2^2 - \rho_3^2 - \lambda_1^2 + \mu_2^2),\\
&\nu_1\nu_2 + \rho_1\rho_2 + \sigma_1\sigma_2,\\
&\nu_1^2 + \rho_1^2 + \sigma_1^2 - \nu_2^2 - \rho_2^2 - \sigma_2^2 + \lambda_1^2 - \mu_2^2,\\
\label{eq:sqrtJ15}
&(\nu_1^2 + \rho_1^2 + \sigma_1^2 - \rho_3^2 - (\sigma_3 + \mu_2)^2 + \lambda_1^2 - \mu_2^2)\notag\\
&\hspace{6\baselineskip}\times(\nu_1^2 + \rho_1^2 + \sigma_1^2 - \rho_3^2 - (\sigma_3 - \mu_2)^2 - \lambda_1^2 + \mu_2^2)
\end{align}
belong to~$\sqrt{J}$, where~$J$ is defined in~\eqref{eq:ideal}.
\end{lemma}

\begin{proof}
See the proof of Lemma~\ref{lem:sqrtJ}.
\end{proof}

\begin{theorem}
\label{thm:crit_real}
A real trifocal tensor~$T = \begin{bmatrix}T_1 & T_2 & T_3\end{bmatrix}$ is calibrated if and only if it satisfies the 15 quartic constraints from Theorems~\ref{thm:nine} and~\ref{thm:six}.
\end{theorem}

\begin{proof}
The ``only if'' part is due to Theorems~\ref{thm:nine} and~\ref{thm:six}. We now prove the ``if'' part.

Let $T = \begin{bmatrix}T_1 & T_2 & T_3\end{bmatrix}$ be a real trifocal tensor satisfying the 15 quartic equations $p_1 = \ldots = p_{15} = 0$. By Lemmas~\ref{lem:transform} and~\ref{lem:Tsimpl}, there exist matrices $Q_1, Q_2, Q_3 \in \mathrm{SO}(3)$ such that the trifocal tensor~$T'$ defined by~\eqref{eq:transform} has form~\eqref{eq:Tsimpl} and satisfies the 15 quartic equations as well.

First we note that if $\lambda_1^2 - \mu_2^2 \neq 0$, then, as the trifocal tensor is real, we get from Lemma~\ref{lem:sqrtJ}:
\[
\lambda_1 = \sigma_1 = \mu_2 = \sigma_2 = 0.
\]
However, this is in contradiction to $\lambda_1^2 - \mu_2^2 \neq 0$. As a result, a real solution to the 15 quartic equations on~$T'$ exists if and only if $\lambda_1^2 = \mu_2^2$. Let us consider two cases: $\rho_3 \neq 0$ and $\rho_3 = 0$.

\medskip
\noindent\textbf{Case 1: $\rho_3 \neq 0$.} Then, $\rho_3^2 + \sigma_3^2 \neq 0$ and, by Lemma~\ref{lem:sqrtJ}, the entries of tensor~$T'$ are constrained by
\begin{multline}
\label{eq:case1constr}
\nu_1\sigma_1 + \nu_2\sigma_2 = \nu_1\rho_1 + \nu_2\rho_2 = \nu_1^2 + \nu_2^2 - \rho_1^2 - \rho_2^2 - \rho_3^2 \\
= (\rho_1\sigma_1 + \rho_2\sigma_2 + \rho_3(\sigma_3 + \mu_2))(\rho_1\sigma_1 + \rho_2\sigma_2 + \rho_3(\sigma_3 - \mu_2)) \\
= (\nu_1^2 + \nu_2^2 - \sigma_1^2 - \sigma_2^2 - (\sigma_3 + \mu_2)^2)(\nu_1^2 + \nu_2^2 - \sigma_1^2 - \sigma_2^2 - (\sigma_3 - \mu_2)^2) = 0.
\end{multline}

The correlation slices of~$T'$ can be represented in form
\[
T'_k = A_2 e_k \begin{bmatrix}0 & 0 & \mu_2\end{bmatrix} - \begin{bmatrix}0 \\ 0 \\ -1\end{bmatrix} e_k^{\mathrm T} A_3^{\mathrm T},
\]
where
\[
A_2 = \begin{bmatrix}\epsilon_1 & 0 & 0\\0 & 1 & 0\\ 0 & 0 & \epsilon_2\end{bmatrix}, \qquad A_3 = \begin{bmatrix}\nu_1 & \nu_2 & 0\\ \rho_1 & \rho_2 & \rho_3\\ \sigma_1 & \sigma_2 & \sigma_3 - \epsilon_2\mu_2\end{bmatrix},
\]
and $\epsilon_k = \pm 1$. It follows that $A_2 = \pm R_2$, where $R_2 \in \mathrm{SO}(3)$. Hence it suffices to show that $A_3 = \theta R_3$, where~$\theta$ is a non-zero scalar and $R_3 \in \mathrm{SO}(3)$. If we suppose that
\[
\rho_1\sigma_1 + \rho_2\sigma_2 + \rho_3(\sigma_3 - \epsilon_2\mu_2) \\= \nu_1^2 + \nu_2^2 - \sigma_1^2 - \sigma_2^2 - (\sigma_3 - \epsilon_2\mu_2)^2 = 0,
\]
then we are done, since due to~\eqref{eq:case1constr} $A_3A_3^{\mathrm T} = \theta^2 I$ with $\theta^2 = \rho_1^2 + \rho_2^2 + \rho_3^2 \neq 0$.

On the other hand, if
\begin{equation}
\label{eq:2eqs}
\rho_1\sigma_1 + \rho_2\sigma_2 + \rho_3(\sigma_3 - \epsilon_2\mu_2) = \nu_1^2 + \nu_2^2 - \sigma_1^2 - \sigma_2^2 - (\sigma_3 + \epsilon_2\mu_2)^2 = 0,
\end{equation}
then we add these polynomials to~$J$ and denote the resulting ideal by~$J_1$. By the computation of the Gr\"{o}bner basis of~$J_1$, we get $(\rho_3\mu_2\sigma_3)^3 \in J_1$. Since $\rho_3 \neq 0$, it follows that either $\mu_2 = 0$ or $\sigma_3 = 0$. In both cases, equalities~\eqref{eq:case1constr} and~\eqref{eq:2eqs} imply $A_3A_3^{\mathrm T} = \theta^2 I$ with $\theta^2 = \rho_1^2 + \rho_2^2 + \rho_3^2 \neq 0$ and hence tensor~$T'$ is calibrated.

\medskip
\noindent\textbf{Case 2: $\rho_3 = 0$.} By Lemma~\ref{lem:sqrtJ1}, the entries of tensor~$T'$ are constrained by (we take into account that $\lambda_1^2 = \mu_2^2$)
\[
(\nu_1^2 + \rho_1^2 + \sigma_1^2)(\sigma_1^2 + \sigma_2^2) = (\nu_2^2 + \rho_2^2 + \sigma_2^2)(\sigma_1^2 + \sigma_2^2) = 0.
\]
Since tensor $T'$ is real, it follows that $\sigma_1 = \sigma_2 = 0$. Again by Lemma~\ref{lem:sqrtJ1} we have
\[
\nu_1\nu_2 + \rho_1\rho_2 = \nu_1^2 + \rho_1^2 - \nu_2^2 - \rho_2^2 = (\nu_1^2 + \rho_1^2 - (\sigma_3 + \mu_2)^2) (\nu_1^2 + \rho_1^2 - (\sigma_3 - \mu_2)^2) = 0.
\]
Suppose first that $\nu_1^2 + \rho_1^2 \neq 0$. Then the correlation slices of~$T'$ can be represented in form
\[
T'_k = A_2 e_k \begin{bmatrix}0 & 0 & \mu_2\end{bmatrix} - \begin{bmatrix}0 \\ 0 \\ -1\end{bmatrix} e_k^{\mathrm T} A_3^{\mathrm T},
\]
where
\[
A_2 = \begin{bmatrix}\epsilon_1 & 0 & 0\\0 & 1 & 0\\ 0 & 0 & \epsilon_2\end{bmatrix}, \qquad A_3 = \begin{bmatrix}\nu_1 & \nu_2 & 0\\ \rho_1 & \rho_2 & 0\\ 0 & 0 & \sigma_3 - \epsilon_2\mu_2\end{bmatrix},
\]
$\epsilon_k = \pm 1$. So that $A_2 = \pm R_2$ and $A_3 = \theta R_3$, where $R_2, R_3 \in \mathrm{SO}(3)$ and $\theta^2 = \nu_1^2 + \rho_1^2 \neq 0$.

Finally, if $\nu_1^2 + \rho_1^2 = 0$, then $\nu_1 = \rho_1 = \nu_2 = \rho_2 = 0$ and~$T'$ is calibrated as well, since
\[
T'_k = \begin{bmatrix}\epsilon_1 & 0 & 0\\0 & 1 & 0\\ 0 & 0 & \epsilon_2\end{bmatrix} e_k \begin{bmatrix}0 & 0 & \mu_2\end{bmatrix} - 0 e_k^{\mathrm T} R_3^{\mathrm T},
\]
where $R_3$ is arbitrary rotation matrix.

Thus we have shown that the trifocal tensor~$T'$ is calibrated in either case. By Lemma~\ref{lem:transform}, the tensor~$T$ is calibrated too. Theorem~\ref{thm:crit_real} is proved.
\end{proof}

\section{Discussion}
\label{sec:disc}

We have defined a new notion --- the trifocal essential matrix. Algebraically, it is a complex $3\times 3$ matrix associated with a given calibrated trifocal tensor~$\hat T$ by the contraction of~$\hat T$ and an arbitrary 3-vector whose squared components sum to zero. Geometrically, it is constructed from a given point on the absolute conic and represents a mapping from the pencil of lines in the third image to the corresponding point in the second image. In this paper, the trifocal essential matrix plays a technical role. However its deeper investigation should help to explain why its properties are so close to the properties of ordinary (bifocal) essential matrix.

Based on the characterization of the set of trifocal essential matrices, we have derived the three necessary conditions on a calibrated trifocal tensor (Theorems~\ref{thm:nine},~\ref{thm:99} and~\ref{thm:six}). They have form of 15 quartic and 99 quintic polynomial equations. We emphasize that these constraints are related to the calibrated case and do not hold for arbitrary trifocal tensors. Moreover, we have shown that the 15 quartic constraints are also sufficient a for a real trifocal tensor to be calibrated. The application of these results to computer vision problems is left for further work.

\bibliographystyle{amsplain}

\end{document}